\documentclass{article}

\usepackage{arxiv}

\usepackage[utf8]{inputenc} 
\usepackage[T1]{fontenc}    
\usepackage{hyperref}       
\usepackage{url}            
\usepackage{booktabs}       
\usepackage{amsfonts}       
\usepackage{nicefrac}       
\usepackage{microtype}      
\usepackage{lipsum}		
\usepackage{graphicx}
\usepackage{natbib}
\usepackage{doi}

\usepackage{amsmath}
\usepackage{amsfonts}
\usepackage{multirow}
\usepackage{amsthm}
\usepackage{cleveref}
\usepackage{threeparttable}
\usepackage{amsthm} 
\newtheorem{thm}{Theorem}
\newtheorem{prop}[thm]{Proposition}

\title{Trust the Prior (or Not):\\Uncertainty-Aware Abdominal Aortic Aneurysm Segmentation}

\author{{Erich Robbi} \\
	Department of Computer Science\\
	University of Trento\\
	Trento, Italy \\
	\texttt{erich.robbi@unitn.it} \\
	\And
	{Daniele Ravanelli} \\
	Medical Physics Unit\\
	Santa Chiara Hospital\\
	Trento, Italy \\\texttt{daniele.ravanelli@asuit.tn.it} \\
	\And
    {Andrea Passerini} \\
	Department of Computer Science\\
	University of Trento\\
	Trento, Italy \\
	\texttt{andrea.passerini@unitn.it} \\
}

\renewcommand{\shorttitle}{Uncertainty-Aware Abdominal Aortic Aneurysm Segmentation}

\hypersetup{
    pdftitle={Trust the Prior (or Not): Uncertainty-Aware Abdominal Aortic Aneurysm Segmentation},
    pdfsubject={Medical Imaging},
    pdfauthor={Erich Robbi, Daniele Ravanelli, Stefano Bonvini, Annalisa Trianni, Andrea Passerini},
    pdfkeywords={Abdominal Aortic Aneurysm, Thrombus Segmentation, Uncertainty Quantification, OOD Generalisation},
}

\begin{document}
\maketitle

\begin{abstract}
	Robust segmentation of intraluminal thrombus is critical for risk assessment in Abdominal Aortic Aneurysm, yet it remains challenging due to heterogeneous thrombus features and low contrast with surrounding non-enhanced tissues. Domain shifts induced by different Computed Tomography Angiography (CTA) protocols further inhibit multi-center generalization of deep learning models. To address these challenges, we propose a patient-specific framework that integrates discriminative learning with anatomically informed priors. Our approach introduces two key components: (1) a patient-specific intensity normalization based on a Gaussian Mixture Model of local anatomy, and (2) an Uncertainty-Gated Anatomical Attention module that incorporates spatial priors while adaptively modulating their influence according to voxel-wise confidence. This design allows for anatomical guidance in ambiguous regions while suppressing unreliable priors. The proposed method achieves state-of-the-art performance on in-distribution test data and substantially outperforms existing alternatives in generalization to external multi-center CTA data, while remaining interpretable through an explicit separation of visual and anatomical evidence.
\end{abstract}


\section{Introduction}
Accurate segmentation of intraluminal thrombus (ILT) in patients with abdominal aortic aneurysm (AAA) is crucial for risk assessment \cite{wanhainen2024editor}. However, obtaining reliable ILT segmentation in pre-operative computed tomography angiography scans (CTAs) poses significant challenges \cite{wang2022fully}, due to insufficient contrast in Hounsfield unit (HU) values between the thrombus and adjacent abdominal structures (Figure \ref{fig:fig1}), leading to indistinct boundaries and a high rate of misclassification \cite{hwang2022automatic}. These challenges are further aggravated by the substantial inter- and intra-patient variability \cite{abdolmanafi2023deep}. 
Deep learning has improved automated ILT segmentation \cite{guo2025automatic,xu2025artificial}. 
Early methods primarily relied on computationally efficient 2D architectures, including FCN- and U-Net-based models, further enhanced through 3D reconstruction, multimodal fusion, hybrid priors, and multi-view/2.5D integration strategies \cite{lopez2018fully,wang2018neural,lopez2017dcnn,caradu2021fully,lareyre2021automated,jung2022abdominal,abdolmanafi2023deep,hwang2022automatic,wang2022fully}, achieving high Dice similarity coefficients (DSC). To enhance volumetric consistency, 3D models, especially those based on 3D U-Net and its variants, were later introduced \cite{kongrat2022reconstruction,mu2023automatic,kim2024computed,lyu2024automatic,robbi2025automatic,zhang2025sdlu}, reporting DSC values between .804 and .987 in well-curated datasets. Furthermore, sophisticated frameworks have tackled the challenges of postoperative aortic repair imaging using Mask R-CNN and Bi-CLSTM architectures \cite{hwang2022automatic,jung2022abdominal}. Nevertheless, purely data-driven models often face difficulties with out-of-distribution (OOD) data \cite{guo2025automatic}. This domain shift arises from center-specific acquisition protocols, reconstruction settings, and contrast timing, which introduce systematic biases in HU distributions \cite {baliyan2019vascular}.
External validation studies have indicated that segmentation errors often arise from misidentification of aneurysmal tissue with respect to adjacent organs \cite{hatzl2024external}. 

\begin{figure*}[t]
\centering
\includegraphics[width=1\textwidth]{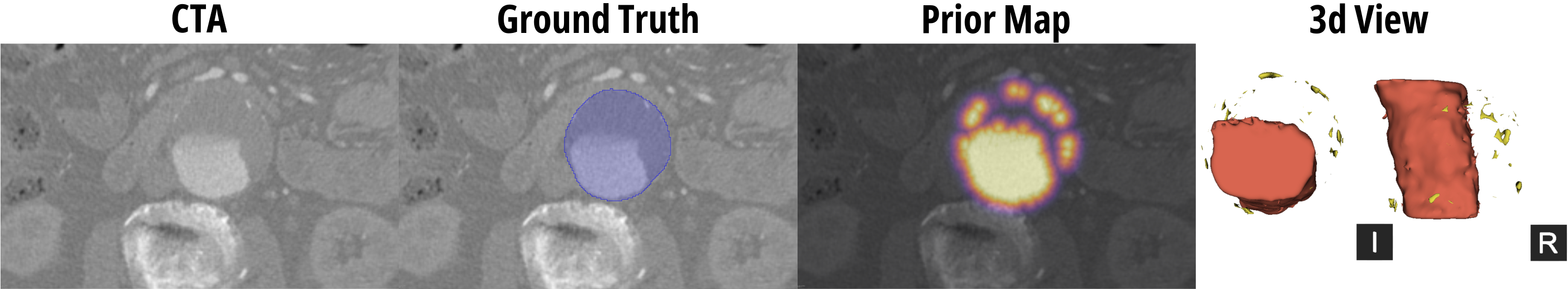}
\caption{CTA in which the thrombus boundary detection is non-trivial, along with the ground truth (blue labeled) and the intuition of using calcifications as positive inductive bias to guide the network shown in axial and 3D views (red: lumen; yellow: calcifications)} \label{fig:fig1}
\end{figure*}

To overcome these challenges, we propose a patient-specific framework combining deep learning with adaptive anatomical priors as shown in Figure \ref{fig:fig1}. Unlike conventional intensity normalisation or fixed anatomical priors, our approach adapts both intensity statistics and spatial constraints at the individual patient level, while explicitly suppressing prior influence under uncertainty. First, we conduct a patient-specific normalisation of HU values obtained from a Gaussian Mixture Model (GMM) representing the local vascular anatomy. Using a pretrained network to automatically detect the contrast-enhanced aortic lumen and the inferior vena cava (IVC), we derive subject-specific HU distributions for vascular structures, which help to minimize inter-scan intensity variability and alleviate domain shift effects. Secondly, we present an Uncertainty-Gated Anatomical Attention (UGAA) module that incorporates spatial constraints based on patient anatomy into the segmentation process. Importantly, the effect of the anatomical prior is adjusted by a probabilistic confidence measure. This design prevents erroneous anatomical bias in challenging regions, allowing the network to rely on learned image features when prior confidence is low. 

We validate our method on the In-distribution (ID) test set and further assess generalizability using a multi-center OOD set of CTAs. Our method achieves (SOTA) performance in comparison to existing techniques, including those provided by data contributors, while exhibiting robustness to OOD CTAs.

In addition to accuracy, the clear breakdown of anatomical evidence and uncertainty offers inherent clinical interpretability. The primary contributions of this research are summarised as follows:

\begin{enumerate}
    \item We introduce a patient-specific deep learning framework for robust intraluminal thrombus segmentation in AAAs from CTA images.
    \item We present an innovative HU normalisation technique based on patient-specific vascular structures to address domain shift issues.
    \item We propose the UGAA module that dynamically incorporates anatomical priors informed by probabilistic confidence levels.
    \item We showcase improved accuracy and generalization across both ID and OOD datasets, while also delivering interpretable anatomical insights.
\end{enumerate}

\section{Methodology}

\begin{figure*}[t]
\centering
\includegraphics[width=1\textwidth]{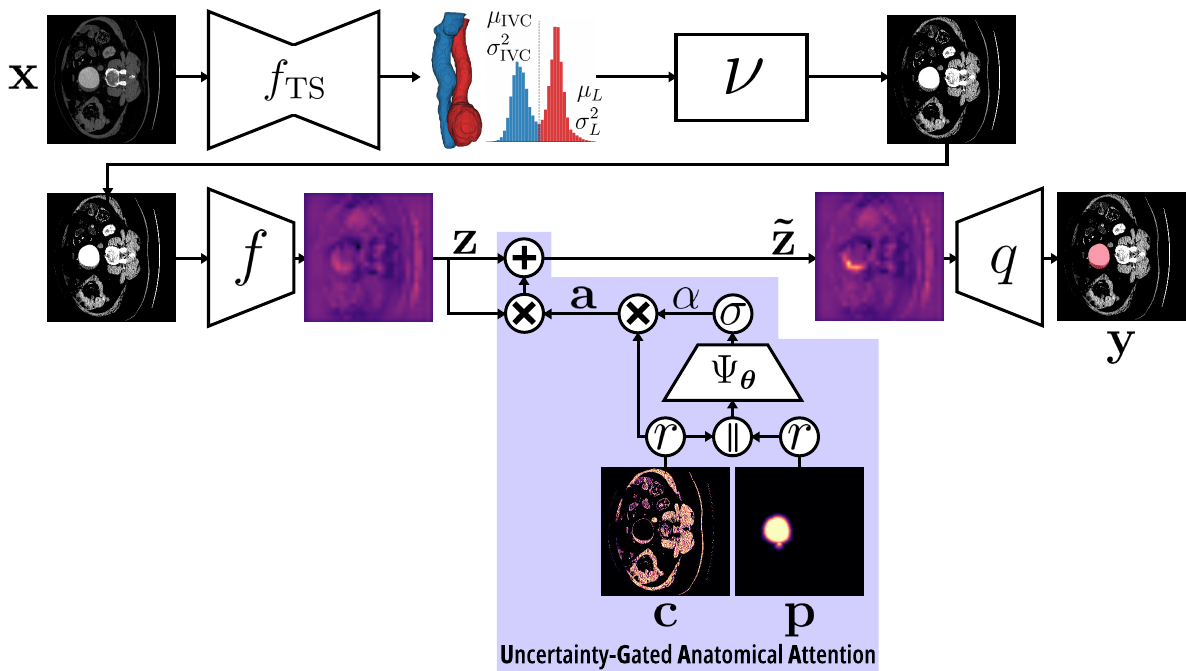}
\caption{Overview of the proposed patient-specific segmentation pipeline. A CTA $\mathbf{x}$ is first normalised via $\nu$, then processed by an encoder–decoder network incorporating the UGAA module at the bottleneck to integrate the anatomical prior $\mathbf{p}$ under uncertainty modulation $\mathbf{c}$.}
\label{fig:pipeline}
\end{figure*}

Our method performs patient-specific segmentation of AAA thrombus in contrast-enhanced CTAs by combining a discriminative segmentation network with anatomically informed priors as shown in
Figure \ref{fig:pipeline}. First, to handle differences between scans and reduce domain shift, we normalize the volume using a linear scaling based on the intensity statistics of non-contrasted tissue ($\nu$) localised by a pre-trained model $f_\text{TS}$ (Figure \ref{fig:pipeline}, first row). We then exploit the predictions of $f_\text{TS}$ to generate and incorporate the anatomical priors $p$ to guide the segmentation, which is dynamically weighted using a confidence map $c$ that considers both overall image quality and local uncertainty. Both $c$ and $p$ are downsampled by the operator $r$. The downsampled spatial prior is injected into the network's latent representations $\textbf{z}$ through an uncertainty-gated anatomical attention module (highlighted in light-blue), allowing the model to rely on prior information only when it is trustworthy (Figure \ref{fig:pipeline}, second row). Finally, to make the results more interpretable, the model’s predictions can be broken down into visual (texture-based) and anatomical (prior-based) evidence, helping clinicians understand what guided each segmentation decision. The different elements of the pipeline are further detailed below. 
\subsection{Patient-specific Normalisation}\label{subsec:norm}
Let $\Omega \subset \mathbb{R}^3$ denote the spatial domain of a CTA volume and $\mathbf{v} \in \Omega$ a voxel location. We represent the input CTA as $\mathbf{x} \in \mathbb{R}^{h \times w \times d}$, where $x_{\mathbf{v}} \in \mathbb{R}$ denotes the intensity at voxel $\mathbf{v}$, and $t \in \mathbb{R}$ a generic intensity value independent of spatial position.

Intensities across CTAs are not directly comparable across patients due to scanner hardware, acquisition protocols, and contrast timing, leading to distribution shifts. In the following, we describe a simple yet effective patient-specific normalisation procedure that is robust to scanner and protocol variability.

We adopt as reference structures 
anatomical regions that are relatively easy to localise across CTAs: the inferior vena cava (IVC) and lumen (L). The IVC, contrary to L, contains non-contrasted blood, so its intensities are similar to those of the thrombus. We assume that there exists a pre-trained model that is able to predict a coarse localisation, and then estimate intensity statistics of both thrombus and contrast medium. TotalSegmentator~\cite{wasserthal2023totalsegmentator}, indicated by $f_\text{TS}$ in our pipeline in Figure \ref{fig:pipeline}, was used to perform this task. Intensities of IVC and L are modeled as a Gaussian Mixture Model (GMM), whose parameters $\mu_{\text{IVC}},\sigma^2_{\text{IVC}}$ and $\mu_{\text{L}},\sigma^2_{\text{L}}$ are estimated via Expectation-Maximisation. We then define the normalisation operator $\nu$ acting on intensity $t$ as 
\begin{equation}
\label{eq:norm}
\nu(t)=\min\!\left(1,\max\!\left[0,\frac{t-(\mu_{\text{IVC}}-2\sigma_{\text{IVC}})}{4\sigma_{\text{IVC}}}\right]\right)
\end{equation}

\begin{prop}
    Assuming thrombus intensities follow the non-contrasted IVC distribution, the operator $\nu$ acts as an affine transformation that maps intensities to a standardized, clipped normal distribution $\mathcal{N}_{[0,1]}\sim\left(\frac{1}{2},\frac{1}{16}\right)$.
\end{prop}
\begin{proof}
Let the intensity $t$ of the non-contrasted tissue follow a normal distribution $\mathcal{N}(\mu_{\text{IVC}}, \sigma^2_{\text{IVC}})$. The normalisation operator $\nu(t)$ is defined in Equation \ref{eq:norm}

Define the linear transformation $Z = \frac{t-(\mu_{\text{IVC}}-2\sigma_{\text{IVC}})}{4\sigma_{\text{IVC}}}$. The expected value and variance of $Z$ are:
$$
\mathbb{E}[Z] = \frac{\mathbb{E}[t] - \mu_{\text{IVC}} + 2\sigma_{\text{IVC}}}{4\sigma_{\text{IVC}}} = \frac{\mu_{\text{IVC}} - \mu_{\text{IVC}} + 2\sigma_{\text{IVC}}}{4\sigma_{\text{IVC}}} = \frac{1}{2} 
$$
$$
 \text{Var}(Z) = \text{Var}\left( \frac{t}{4\sigma_{\text{IVC}}} \right) = \frac{\sigma^2_{\text{IVC}}}{16\sigma^2_{\text{IVC}}} = \frac{1}{16}$$
Thus, $Z$ follows $\mathcal{N}(\frac{1}{2}, \frac{1}{16})$. The clipping operator $\min(1, \max(0, Z))$ maps the range $[\mu_Z - 2\sigma_Z, \mu_Z + 2\sigma_Z] = [0, 1]$ to the unit interval, effectively normalizing the intensity distribution across varying clinical acquisition protocols.
\end{proof}

This effectively enforces conditional domain invariance, significantly reducing inter-patient intensity variability.

\subsection{Anatomical Prior}\label{subsec:prior}
In the context of AAA the thrombus typically surrounds the lumen, but that is not the only useful feature: calcifications (C) may also be present along the boundaries of the sac, a feature recognized in clinical guidelines \cite{wanhainen2024editor}. As these calcifications have very high HU values, they distinguish thrombus from surrounding tissues, defining the true boundaries of the aneurysm. Using the previously introduced GMM and a defined region of interest $\Omega_{\text{ROI}} \subset \Omega$ (either between the kidneys and the iliac arteries or within centimeters from the localised aorta), we generate an anchor set $M = \{ \mathbf{v} \in \Omega_{\text{ROI}}\mid x_\mathbf{v} \ge \tau^*\}$ that highlights both L and C near the aneurysm. Here $\tau^*$ is the minimum HU value separating the IVC and L  between $\mu_{\text{IVC}}$ and $\mu_{\text{L}}$.
With the anchor set we compute the prior $p$ for the thrombus via a euclidean distance transform:
\begin{equation}
\label{eq:prior}
\begin{aligned}
p(\mathbf{v})
&=
\exp\!\left(
-\frac{d(\mathbf{v},M)^2}{2\sigma^2_{\text{spatial}}}
\right)
\quad \forall \mathbf{v} \in \Omega \\ &\text{where} \quad
d(\mathbf{v},M)
=
\min_{\mathbf{m}\in M}
\|\mathbf{v}-\mathbf{m}\|_2
\end{aligned}
\end{equation}
This reflects the idea that the likelihood of a voxel being thrombus increases as $d(\mathbf{v},M)\rightarrow 0$.
Defining the prior as in \eqref{eq:prior} does not account for the intensity of individual voxels. As a result, a voxel located very close to the prior map may still be assigned a high likelihood even if its intensity is inconsistent with non-contrast blood. Moreover, when there is no clear separation between contrast and non-contrast intensities, a voxel lying close to the mean of the non-contrast distribution may still be ambiguously classified due to significant aleatoric uncertainty between the two distributions. These limitations can make the prior unreliable.
We introduce a patient-level confidence measure which modulates the prior based on global image quality and local voxel ambiguity:
\begin{equation}
\begin{aligned}
c_g &= 1 - \exp\!\left(
-\frac{(\mu_{\text{IVC}} - \mu_{\text{L}})^2}
{2(\sigma_{\text{IVC}}^2 + \sigma_{\text{L}}^2)}
\right) \\
c_l(\mathbf{v}) &=
\frac{\mathcal{N}_{\text{IVC}}(x_\mathbf{v})}
     {\sum_{k \in \{\text{IVC}, \text{L}\}} \mathcal{N}_k(x_\mathbf{v})}
\cdot
\frac{\mathcal{N}_{\text{IVC}}(x_\mathbf{v})}
     {\max_{\mathbf{v} \in \Omega} \mathcal{N}_{\text{IVC}}(x_\mathbf{v})}
\end{aligned}
\end{equation}

where
\[
\mathcal{N}_k(x_\mathbf{v})
:= \mathcal{N}(x_\mathbf{v}\mid\mu_k,\sigma_k^2),
\qquad
k\in\{\text{IVC},\text{L}\}.
\]
The confidence measure $c(\mathbf{v})=c_g\cdot c_l(\mathbf{v})$ quantifies how reliable the spatial prior is for each voxel. The global term $c_g$ captures how well the non-contrast and contrast intensity distributions are separated across the whole image, 
while the local term $c_l(\mathbf{v})$ checks whether the intensity at voxel $\mathbf{v}$ matches the expected non-contrast profile.
Thus, $c(\mathbf{v})$ is only high when the prior is supported by the global image quality and the local intensity. When $c(\mathbf{v}) \approx 0$, the prior can be effectively ignored, preventing it from misleading the model.
 Although prior and confidence are defined at the voxel level, they are applied to the whole CTA.

The voxel-level functions $p(\cdot)$ and $c(\cdot)$ are applied across the entire CTA to produce volumetric prior and confidence maps. For notational simplicity, for a given CTA $\mathbf{x}$, we write $\mathbf{p} = p(\mathbf{x})$ and $\mathbf{c} = c(\mathbf{x})$.

\subsection{Uncertainty-Gated Anatomical Attention}
We propose the UGAA module to inject the spatial prior into the latent representation between the encoder and decoder of the model, while strictly conditioning its influence on its estimated reliability.
In the standard ungated architecture, the encoder $f$ maps the input CTA $\textbf{x} \in \mathbb{R}^{h \times w \times d}$ to a latent representation $\textbf{z} = f(\textbf{x}) \in  \mathbb{R}^{k \times h' \times w' \times d'}$ which is processed by the decoder to give the final prediction $\textbf{y} = q(\textbf{z}) \in [0,1]^{h \times w \times d}$.
Given that the prior and confidence maps are defined at the input (rather than latent) level, we introduce an operator $r:\mathbb{R}^{h \times w \times d} \rightarrow \mathbb{R}^{h' \times w' \times d'}$ which aligns the resolution of the input with the resolution of the latent space. Practically, $r$ denotes the trilinear operator to the latent resolution.
We then define a convolutional layer $\Psi_{\boldsymbol{\theta}}$ parametrized by $\boldsymbol{\theta}$ which maps the prior and confidence maps to attention coefficients in the latent space
$\alpha(\textbf{p}, \textbf{c}) = \sigma \left( \Psi_{\boldsymbol{\theta}}([r(\textbf{p}) || r(\textbf{c})])\right)  \in [ 0,1 ]^{h' \times w' \times d'}$,
where $r(\textbf{p})$ and $r(\textbf{c})$ are the downsampled prior and confidence maps, respectively, $||$ is the element-wise concatenation operator and $\sigma$ denotes the sigmoid function. 
We don't want to inject into the latent representation an unreliable prior (i.e. $\textbf{c} \approx \textbf{0}$), as this can possibly inject wrong information that could lead the model to make erroneous predictions. For this reason we define the effective attention map $\textbf{a}$ by gating the attention coefficient with the downsampled confidence score:
$\textbf{a} = \alpha(\textbf{p},\textbf{c}) \odot r(\textbf{c})$
where $\odot$ is the Hadamard operator. We can now compute prior-informed latent representations as:
\begin{equation}\label{eq:zt}
\tilde{\textbf{z}}_{i} = \textbf{z}_{i} \odot \big( 1 + \textbf{a} \big) \quad i \in \{1, \dots, k\}
\end{equation}
The definition of $\tilde{\mathbf{z}}$ makes it possible to generate prior influence maps, showing exactly where and how much the anatomical prior affected the network’s decision.
With the following proposition we explain that module defaults to standard processing when priors are unreliable, ensuring safe behavior in clinical settings.
\begin{prop}
    The perturbation between the gated features $\tilde{\mathbf{z}}$ and bottleneck features $\mathbf{z}$ is bounded channel-wise by the downsampled confidence map $\|\tilde{\mathbf{z}}_i - \mathbf{z}_i\| \;\le\; \|\mathbf{z}_i\| \cdot \|r(\mathbf{c})\| \quad \forall i\in\{1,\dots,k\}$.
\end{prop}
\begin{proof}
    From Equation \ref{eq:zt} we have that the prior-informed latent representation is 
    $$
    \tilde{\textbf{z}}_{i} = \textbf{z}_{i} \odot \big( 1 + \textbf{a} \big)
    $$
    By isolating the perturbation and subtracting the original latent representation $\mathbf{z}_i$ from both sides we obtain
    $$
    \mathbf{\tilde{z}}_i - \mathbf{z}_i = \mathbf{z}_i \odot \textbf{a} \Rightarrow \mathbf{\tilde{z}}_i - \mathbf{z}_i = \mathbf{z}_i \odot (\alpha(\textbf{p},\textbf{c}) \odot r(\textbf{c}))
    $$
    We take the norm, and since $0 \le \alpha(\textbf{p},\textbf{c}) \le 1$, we can establish an upper bound, dropping $\alpha$ from the product
    $$
        \|\tilde{\mathbf{z}}_i - \mathbf{z}_i\| \;\le\; \|\mathbf{z}_i \odot r(\textbf{c)}\|
    $$
    
    Finally, by the Hadamard-product inequality,
    $$
    \|\mathbf{z}_i\odot r(\mathbf{c})\|
    \le
    \|\mathbf{z}_i\|\,\|r(\mathbf{c})\|
    $$
    
    we have
    $$
    \|\tilde{\mathbf{z}}_i-\mathbf{z}_i\|
    \le
    \|\mathbf{z}_i\|\,\|r(\mathbf{c})\|
    $$
\end{proof}
Consequently, if the confidence map vanishes (i.e. $r(\mathbf{c})=\mathbf{0}$), UGAA defaults to an identity mapping $\tilde{\mathbf{z}}=\mathbf{z}$, recovering the baseline architecture. Thus, unreliable priors can be completely suppressed rather than forcibly injected into the latent representation.

Beyond inference guarantees, UGAA promotes training stability by modulating gradient flow:
\begin{prop}
    The gradient of the loss with respect to the attention parameters $\boldsymbol{\theta}$ is explicitly scaled by the confidence map:
\[\frac{\partial \mathcal{L}}{\partial \boldsymbol{\theta}} = \left( \left[ \sum_{i=1}^k \frac{\partial \mathcal{L}}{\partial \tilde{\mathbf{z}}_i} \odot \mathbf{z}_i \right] \odot r(\mathbf{c}) \right) \frac{\partial \alpha(\textbf{p},\textbf{c})}{\partial \boldsymbol{\theta}}
\]
so that when confidence is low ($r(\mathbf{c}) \to \mathbf{0}$), the magnitude of gradient updates vanishes.
\end{prop}
\begin{proof}
Given the gated feature representation $\tilde{\mathbf{z}}_i = \mathbf{z}_i \odot (1 + \mathbf{a})$, where $\mathbf{a} = \alpha(\mathbf{p}, \mathbf{c}) \odot r(\mathbf{c})$, we derive the gradient of the loss $\mathcal{L}$ with respect to the attention parameters $\boldsymbol{\theta}$.

Using the chain rule, the gradient is given by:
$$
\frac{\partial \mathcal{L}}{\partial \boldsymbol{\theta}} = \sum_{i=1}^k \frac{\partial \mathcal{L}}{\partial \tilde{\mathbf{z}}_i} \odot \frac{\partial \tilde{\mathbf{z}}_i}{\partial \boldsymbol{\theta}}
$$
First, we compute the partial derivative of the feature update $\tilde{\mathbf{z}}_i$ with respect to the attention map $\mathbf{a}$:
$$
\frac{\partial \tilde{\mathbf{z}}_i}{\partial \mathbf{a}} = \frac{\partial}{\partial \mathbf{a}} \left( \mathbf{z}_i + \mathbf{z}_i \odot \mathbf{a} \right) = \mathbf{z}_i
$$
Next, we calculate the partial derivative of the attention map $\mathbf{a}$ with respect to the parameters $\boldsymbol{\theta}$:
$$
\frac{\partial \mathbf{a}}{\partial \boldsymbol{\theta}} = \frac{\partial}{\partial \boldsymbol{\theta}} \left( \alpha(\mathbf{p}, \mathbf{c}) \odot r(\mathbf{c}) \right) = \frac{\partial \alpha(\mathbf{p}, \mathbf{c})}{\partial \boldsymbol{\theta}} \odot r(\mathbf{c})
$$
Substituting these back into the primary chain rule expression and applying the associativity of the Hadamard product:
$$
\begin{aligned}
\frac{\partial \mathcal{L}}{\partial \boldsymbol{\theta}} &= \sum_{i=1}^k \left( \frac{\partial \mathcal{L}}{\partial \tilde{\mathbf{z}}_i} \odot \frac{\partial \tilde{\mathbf{z}}_i}{\partial \mathbf{a}} \right) \odot \frac{\partial \mathbf{a}}{\partial \boldsymbol{\theta}} \\
&= \sum_{i=1}^k \left( \frac{\partial \mathcal{L}}{\partial \tilde{\mathbf{z}}_i} \odot \mathbf{z}_i \right) \odot \left( \frac{\partial \alpha(\mathbf{p}, \mathbf{c})}{\partial \boldsymbol{\theta}} \odot r(\mathbf{c}) \right) \\
&= \left( \left[ \sum_{i=1}^k \frac{\partial \mathcal{L}}{\partial \tilde{\mathbf{z}}_i} \odot \mathbf{z}_i \right] \odot r(\mathbf{c}) \right) \odot \frac{\partial \alpha(\mathbf{p}, \mathbf{c})}{\partial \boldsymbol{\theta}}
\end{aligned}
$$
\end{proof}
This effectively prevents the module from learning from unreliable regions or overfitting to noise. 

\section{Experiments}
We address the following research questions:
\begin{itemize}
    \item[] \textbf{Q1} Does normalisation reduce domain shift?
    \item[] \textbf{Q2} Does UGAA improve segmentation?
    \item[] \textbf{Q3} Does the complete framework outperform existing methods?
\end{itemize}
\paragraph{Datasets} We evaluate on three CTA sets: ID (200 train / 20 test scans, single center, single vendor and scanner, 1.25 mm slice thickness) \cite{siriapisith2022retrospective}, OOD (26 scans, three centers, three vendors, five scanners, 0.62–3.0 mm slice thickness, from 1996 onwards) \cite{clark2013cancer,radl2022avt,wilson2013vascular}, and Clinical (50 scans, nine centers, three vendors, nine scanners, 0.62–2.0 mm slice thickness, 2007–2026), approved by the Ethics Committee (EC approval no. A1044; protocol code: TANGO). 
\paragraph{Architectures} To ensure a comprehensive evaluation, we trained and assessed a heterogeneous set of architectures representing different methodological paradigms. These span standard encoder-decoder convolutional networks (U-Net \cite{unet}, V-Net \cite{vnet}), residual and dense architectures (SegResNet \cite{segresnet}, ResNetMed \cite{resnetmed3d}, DenseVoxNet \cite{densevoxnet}), attention-based and transformer models (AG-DSV-UNet \cite{agdsv}, Swin-UNETR \cite{swinunetr}), as well as the auto-configuring SOTA model nn-UNet \cite{nnunet} and a recent domain-specific method of Siriapisith et al. \cite{siriapisith2022retrospective}.

For U-Net, AG-DSV-UNet, ResNetMed, DenseVoxNet, V-Net, and the proposed method of Siriapisith et al., the hyperparameter settings reported in that work were adopted to ensure a faithful reproduction of the original experimental setup. For nn-UNet, the default self-configured settings were used. Swin-UNETR and SegResNet were trained with AdamW optimization (initial learning rate $10^{-3}$, weight decay $10^{-5}$) and DiceCE loss.

\paragraph{A1: The normalisation helps in reducing domain shift}
To evaluate the efficacy of the proposed normalisation method against domain shift, we computed pairwise Wasserstein distance (W-dist) between intensity distributions of the segmented thrombus across patients within the ID, OOD and Clinical cohorts. As Table \ref{tab:norm} shows, in the original HU space, the OOD cohort exhibits a higher baseline discrepancy in mean W-dist compared to the ID cohort (16.66 vs. 10.91), while the Clinical cohort presents the largest distributional variability overall (25.02). Following the proposed normalisation, the mean pairwise distance collapsed substantially across all cohorts, reducing to 0.05, 0.06, and 0.09 for the ID, OOD, and Clinical sets, respectively. Consequently, the large intensity discrepancies observed before normalisation were markedly reduced, despite the heterogeneous nature of the Clinical cohort. The variance between scans also stabilised considerably, with the standard deviation decreasing from 7.64, 9.68, and 18.17 in the original HU space to 0.03, 0.03, and 0.06 after normalisation for the ID, OOD, and Clinical cohorts, respectively.
\begin{table}
\centering
\caption{Domain invariance evaluation using pairwise Wasserstein distances. The mean, standard deviation, and 95\% bootstrap confidence intervals are reported for the original Hounsfield Unit (HU) space and our normalised [0,1] space across ID and OOD splits.}\label{tab:norm}
{
\setlength{\tabcolsep}{2pt}
\begin{tabular}{lccc}
\toprule
\textbf{Set} & \textbf{Mean $\pm$ Std} & \textbf{95\% Bootstrap C.I.} & \textbf{\# Pairs} \\
\midrule
\multicolumn{4}{c}{\textbf{Before Normalisation}} \\
\midrule
 ID & 10.91 $\pm$ 7.64 & [10.80, 11.02] & 19900 \\
 OOD & 16.66 $\pm$ 9.68 & [15.60, 17.76] & 325 \\
 Clinical & 25.02 $\pm$ 18.17 & [24.00, 26.04] & 1225 \\

\midrule
\multicolumn{4}{c}{\textbf{After Normalisation}} \\
\midrule

 ID & 0.05 $\pm$ 0.03 & [0.05, 0.05] & 19900 \\
 OOD & 0.06 $\pm$ 0.03 & [0.06, 0.07] & 325 \\
 Clinical & 0.09 $\pm$ 0.06 & [0.08, 0.09] & 1225 \\

\bottomrule

\end{tabular}
}
\end{table}
\paragraph{A2: Anatomical priors and uncertainty gating improve segmentation robustness} We investigate how each component of the proposed framework contributes to performance under both ID and OOD settings. In particular, we evaluate (i) patient-specific normalisation, (ii) anatomical priors, and (iii) UGAA by conducting an ablation study (Table \ref{tab:ablation_combined}). 
In the ID setting, all configurations perform similarly, with DSC values of 0.974 (baseline), 0.974 (+$\nu$), 0.976 (+$p$), and 0.974 (+UGAA), and corresponding HD95 values of 6.6, 1.3, 1.2, and 1.3. 
In contrast, the OOD setting reveals substantial performance differences. While normalisation alone provides negligible improvement over the baseline (DSC 0.763 $\rightarrow$ 0.768; HD95 51.0 $\rightarrow$ 16.2), incorporating anatomical priors leads to a pronounced gain in segmentation accuracy (DSC 0.768 $\rightarrow$ 0.911; HD95 16.2 $\rightarrow$ 9.0). 
Finally, the introduction of UGAA further improves OOD performance (DSC 0.911 $\rightarrow$ 0.938; HD95 9.0 $\rightarrow$ 5.6). This gain is achieved while maintaining stable ID performance, indicating that selectively modulating prior influence based on confidence prevents degradation in ambiguous or unreliable anatomical regions. 


\begin{table}\label{tab:ablation_combined}
\centering
\caption{Ablation study results. Starting from a baseline, we progressively integrated the components in the following order: 1) patient-specific normalisation, 2) the anatomical prior, and 3) the complete UGAA module}
\begin{tabular}{lcccc}
\toprule
& \multicolumn{2}{c}{\textbf{In-Distribution}} & \multicolumn{2}{c}{\textbf{Out-of-Distribution}} \\
\cmidrule(lr){2-3} \cmidrule(lr){4-5}
\textbf{Method} & \textbf{DSC} $\uparrow$ & \textbf{HD95} $\downarrow$ & \textbf{DSC} $\uparrow$ & \textbf{HD95} $\downarrow$ \\
\midrule

Baseline & .924 $\pm$ 0.014 & 6.6 $\pm$ 2.9 & .763 $\pm$ 0.101 & 51.0 $\pm$ 18.7 \\
+ $\nu$  & .974 $\pm$ 0.004 & 1.3 $\pm$ 0.1 & .768 $\pm$ 0.123 & 16.2 $\pm$ 8.7 \\
+ $p$    & \textit{\textbf{.976 $\pm$ 0.003}} & \textit{1.2 $\pm$ 0.1} & .911 $\pm$ 0.037 & 9.0 $\pm$ 6.4 \\
+ UGAA   & .974 $\pm$ 0.004 & 1.3 $\pm$ 0.1 & \textit{\textbf{.938 $\pm$ 0.026}} & \textit{\textbf{5.6 $\pm$ 5.1}} \\

\bottomrule
\end{tabular}
\begin{tablenotes}
\small
\item Note: Best performance highlighted in italic. Boldface indicates statistical significant improvement of the best performing method over the runner up, assessed using the Wilcoxon signed-rank test ($p<0.01$)
\end{tablenotes}
\end{table}
\paragraph{A3: UGAA performs either equally or greater than SOTA baselines }
We use DSC to measure overall model performance, whereas to quantify accuracy around the boundaries we instead use the 95\% Hausdorff distance (HD95).
Each method was systematically evaluated across the ID, OOD, and Clinical cohorts to assess both absolute segmentation accuracy and robustness against domain shift. Visual comparisons are shown in Figure \ref{fig:fig_results}.

In the ID setting, all approaches achieve high segmentation accuracy with comparable DSC values, ranging from 0.910 to 0.971. Still, our method achieves the highest DSC (0.974) together with nn-UNet (0.971), and consistently yields improved boundary accuracy, reaching an HD95 of 1.3 compared to 1.6 for nn-UNet and higher values for all other baselines. However, improvements over the second best architecture (nn-UNet) are not statistically significant.

In the OOD setting, all methods exhibit a performance degradation. In particular, DSC drops substantially for most methods (e.g., 0.888 for VNet and 0.916 for Swin-UNETR), whereas our framework maintains the highest performance (DSC 0.937) with a significantly lower HD95 of 5.6 compared to 14.6 (Swin-UNETR), 18.7 (VNet), and 35.6 (Siriapisith et al.). This decline is substantially more pronounced in all approaches with the exception of Swin-UNETR and UGAA, but our framework maintains significantly higher segmentation accuracy and better boundary delineation ($p < 0.001$). 

Evaluation on the Clinical dataset further highlights the generalization capability of the proposed method, even in the presence of heavy artifacts (Figure \ref{fig:fig_results} last row). In this highly heterogeneous setting, our framework achieves the highest mean DSC (0.953), outperforming VNet (0.936), UNet (0.931), and Siriapisith et al. (0.929), while also achieving the lowest HD95 (1.8 vs. 2.8, 11.6, and 6.1 respectively). More importantly, it shows substantially reduced variance compared to all other baselines, indicating more stable predictions across centers and scanner models.

\begin{figure*}[t]
\centering
\includegraphics[width=1\textwidth]{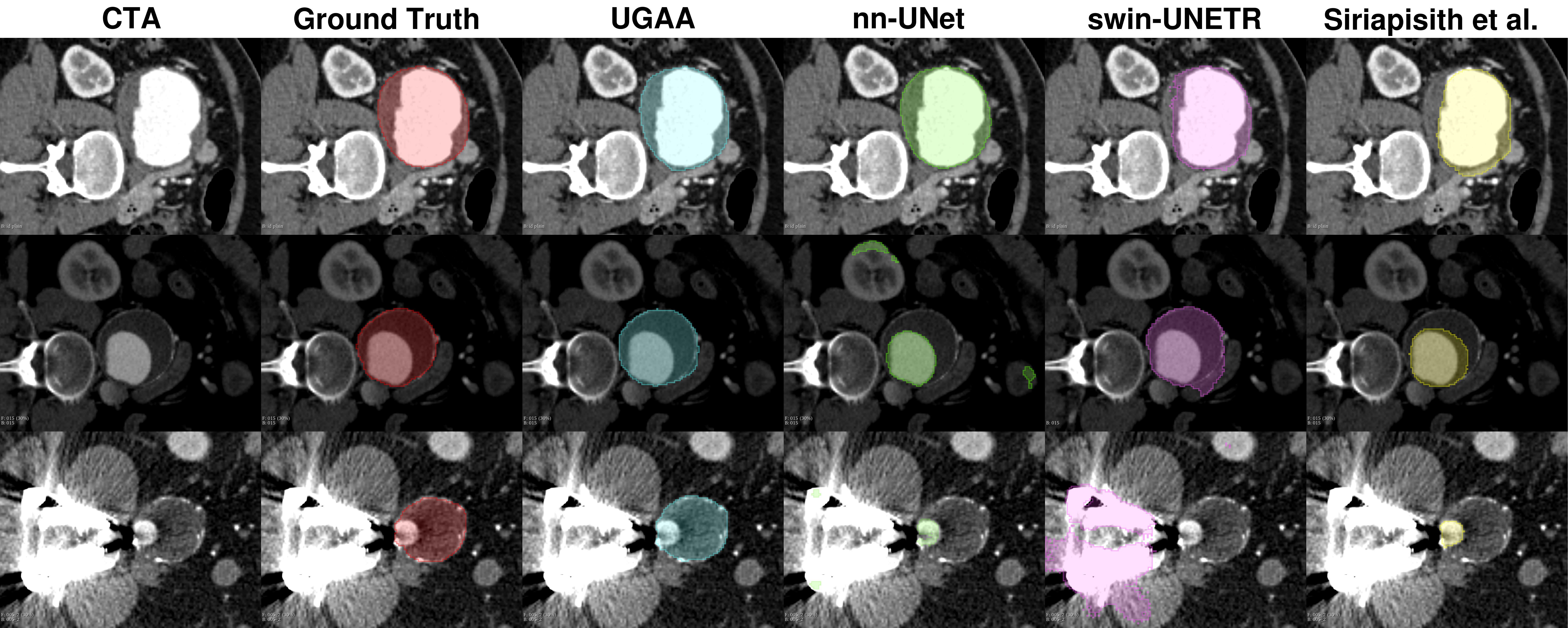}
\caption{Qualitative comparison on one CTA case from each cohort. The first, second, and third rows correspond to the ID, OOD, and Clinical cohorts, respectively. Red contours denote the ground truth, cyan contours denote the predictions of our method, and the remaining colors correspond to the predictions of the best-performing baseline in each cohort (nn-UNet for the ID cohort, Swin-UNETR for the OOD cohort, and Siriapisith et al. for the Clinical cohort).}\label{fig:fig_results}
\end{figure*}

\begin{figure*}[t]
\centering
\includegraphics[width=1\textwidth]{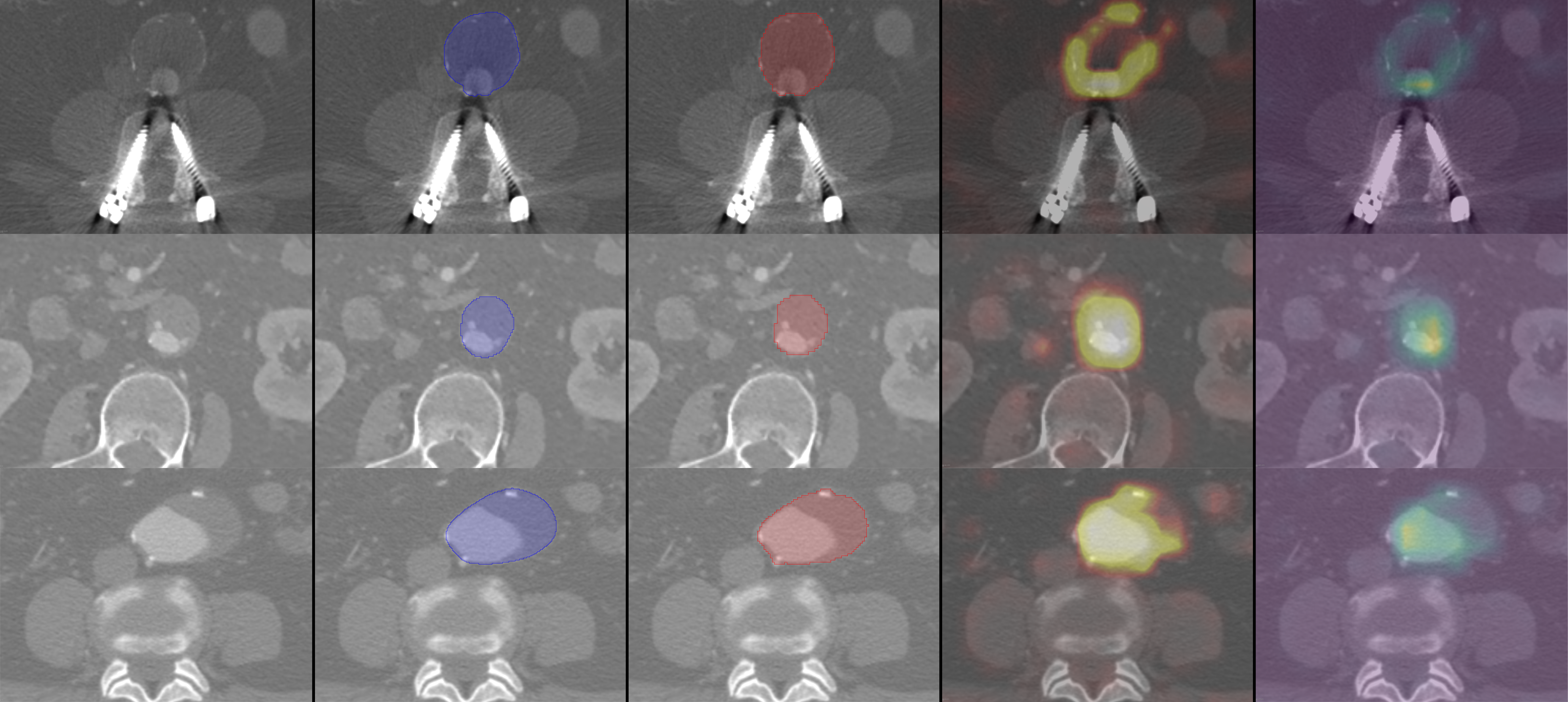}
\caption{Qualitative analysis of the interpretability of UGAA on representative clinical CTA cases. Blue contours denote the ground truth, while red contours denote the predictions of our method. The fourth column shows the attention map produced by the UGAA module, and the last column shows the prior influence map, highlighting the regions where the prior contributed to the prediction and the magnitude of its contribution.}\label{fig:fig_interpretability}
\end{figure*}
\begin{table}
\centering
\caption{Test sets results across ID, OOD, and Clinical sets}\label{tab:perf}
{
\setlength{\tabcolsep}{2pt}
\begin{tabular}{clcc}
\toprule
\textbf{Set} & \textbf{Method} & \textbf{DSC} $\uparrow$ & \textbf{HD95} $\downarrow$ \\
\midrule

\multirow{10}{*}{\shortstack{ID\\$n=20$}} 
& UNet & .942 $\pm$ 0.023 & 6.3 $\pm$ 4.8 \\
& SegResNet & .924 $\pm$ 0.014 & 6.6 $\pm$ 2.9 \\
& AG-DSV-UNet & .945 $\pm$ 0.015 & 9.1 $\pm$ 7.9 \\
& ResNetMed & .910 $\pm$ 0.019 & 5.5 $\pm$ 1.7 \\
& DenseVoxNet & .925 $\pm$ 0.016 & 17.0 $\pm$ 9.9 \\
& VNet & .962 $\pm$ 0.007 & 2.2 $\pm$ 0.6 \\
& Swin-UNETR & .950 $\pm$ 0.014 & 5.0 $\pm$ 3.4 \\
& nn-UNet & .971 $\pm$ 0.005 & 1.6 $\pm$ 0.4 \\
& Siriapisith et al. & .966 $\pm$ 0.008 & 6.8 $\pm$ 7.2 \\
& UGAA & \textit{.974} $\pm$ \textit{0.004} & \textit{1.3} $\pm$ \textit{0.1} \\
\midrule

\multirow{10}{*}{\shortstack{OOD\\$n=26$}} 
& UNet & .825 $\pm$ 0.057 & 47.7 $\pm$ 17.2 \\
& SegResNet & .763 $\pm$ 0.103 & 51.0 $\pm$ 18.5 \\
& AG-DSV-UNet & .667 $\pm$ 0.085 & 51.3 $\pm$ 13.5 \\
& ResNetMed & .731 $\pm$ 0.088 & 37.3 $\pm$ 13.2 \\
& DenseVoxNet & .778 $\pm$ 0.068 & 67.5 $\pm$ 15.9 \\
& VNet & .888 $\pm$ 0.044 & 18.7 $\pm$ 11.8 \\
& Swin-UNETR & .916 $\pm$ 0.022 & 14.6 $\pm$ 8.1 \\
& nn-UNet & .721 $\pm$ 0.142 & 69.7 $\pm$ 37.4 \\
& Siriapisith et al. & .890 $\pm$ 0.042 & 35.6 $\pm$ 16.3\\
& UGAA & \textit{\textbf{.937} $\pm$ \textbf{0.027}} & \textit{\textbf{5.6} $\pm$ \textbf{5.1}}\\
\midrule

\multirow{10}{*}{\shortstack{Clinical\\$n=50$}}  
& UNet & .931 $\pm$ 0.020 & 11.6 $\pm$ 6.6 \\
& SegResNet &.890 $\pm$ 0.025 & 18.5 $\pm$ 9.2 \\
& AG-DSV-UNet & .895 $\pm$ 0.030 &  11.4 $\pm$ 6.0 \\
& ResNetMed & .828 $\pm$ 0.047 & 17.2 $\pm$ 8.2 \\
& DenseVoxNet & .854 $\pm$ 0.047 & 36.3 $\pm$ 10.5 \\
& VNet & .936 $\pm$ 0.019 & 2.8 $\pm$ 1.0 \\
& Swin-UNETR & .917 $\pm$ 0.019 & 15.0 $\pm$ 5.8 \\
& nn-UNet & .870 $\pm$ 0.038 & 5.6 $\pm$ 1.7 \\
& Siriapisith et al. & .929 $\pm$ 0.038 & 6.1 $\pm$ 4.8\\
& UGAA & \textit{\textbf{.953} $\pm$ \textbf{0.008}} & \textit{\textbf{1.8} $\pm$ \textbf{0.4}}\\

\bottomrule

\end{tabular}
}
\begin{tablenotes}
\small
\item Note: Best performance highlighted in italic. Boldface indicates statistical significant improvement of the best performing method over the runner up, assessed using the Wilcoxon signed-rank test ($p<0.01$)
\end{tablenotes}
\end{table}

\section{Discussion}
The normalisation compresses inter-subject intensity differences by two orders of magnitude in Wasserstein space. These findings suggest that the proposed patient-specific normalisation substantially reduces inter-patient intensity variability, which helps to achieve near-optimal performance in a homogeneous imaging setting. 
However, as reported in the ablation study in the OOD set, patient-specific normalisation alone is insufficient. The additional components provide substantial gains, while the effect of normalisation remains limited. 
Results highlight that anatomical priors are the main driver of OOD robustness: the UGAA module further enhances stability by modulating prior influence based on uncertainty, preventing degradation when anatomical cues are unreliable. This suggests improved robustness to variations in slice thickness and anatomical regions, which introduce locations not observed during training.

Although calcifications are highly prevalent in AAA patients \cite{maier2010impact,bartstra2021abdominal} and provide strong cues to the model due to their high attenuation values, our proposed framework does not exclusively rely on their presence. Specifically, the anchor set $M$ includes both calcifications, if present, and the contrast-enhanced lumen, defining a solid prior even in patients without visible calcifications. Additionally, the confidence-gated formulation in Section \ref{subsec:prior} reduces the prior influence whenever there is uncertainty in the anatomy. 

One limitation of the proposed work is the dependence on TotalSegmentator. 
In fact, the normalisation and prior generation steps depend on the presence of lumen and IVC segmentation. Although TotalSegmentator demonstrated robust performance across diverse imaging settings \cite{wasserthal2023totalsegmentator} and for our work we don't need precise segmentations from it, errors may propagate to the downstream stages of the pipeline. However, in our medium-sized dataset we never encountered such a problem. 

Regarding the dataset size, although $\approx 100$ CTA scans is a typical dataset size in the ILT literature \cite{guo2025automatic}, and the OOD and clinical cohorts encompass multiple vendors, scanners, acquisition protocols, and decades of imaging technology, larger multi-center studies are required to further assess performance across rarer aneurysm morphologies. The results of this study are nevertheless encouraging and support further validation in larger cohorts.

An important aspect of the proposed framework is its contribution to interpretability within a clinical workflow, particularly for surgical planning in abdominal AAA cases. Beyond producing a binary thrombus segmentation, the model explicitly separates visual evidence from anatomically driven priors through the UGAA mechanism, enabling a more structured explanation of predictions, as Figure \ref{fig:fig_interpretability} shows. In practice, this allows clinicians to inspect not only the final segmentation, but also the underlying spatial prior map and the corresponding confidence modulation that determined where anatomical guidance was trusted or suppressed. For surgeons, this distinction is particularly relevant in pre-operative planning, where understanding why a region is classified as thrombus can be just as important as the classification itself. The prior influence map highlights anatomical proximity to key vascular structures such as the lumen and calcified regions, providing an intuitive geometric reference that aligns with established clinical reasoning. At the same time, the confidence map acts as a quality indicator, signaling regions where the model relies primarily on learned image features due to ambiguous or unreliable anatomical cues. This dual representation supports a more transparent decision-making process, potentially increasing trust in automated outputs and facilitating their integration into multidisciplinary surgical pipelines. In complex or borderline cases, such as heavily calcified or low-contrast aneurysms, this interpretability may also assist clinicians in identifying regions where manual review is warranted, thereby positioning the model as a decision-support tool rather than a standalone predictor.

\section{Conclusions}
This work investigated three research questions concerning robust intraluminal thrombus segmentation in abdominal aortic aneurysms from CTA images.

The proposed normalisation technique substantially reduces domain shift. By aligning vascular intensity distributions using subject-specific anatomical references, the proposed normalisation reduced inter-patient Wasserstein distances by more than two orders of magnitude across ID, OOD, and Clinical cohorts, thereby improving intensity consistency across heterogeneous acquisition protocols.

Furthermore, we showed that anatomical priors significantly improve segmentation robustness, particularly in OOD settings. The proposed module further improved performance by selectively incorporating prior information only when supported by sufficient confidence, preventing unreliable anatomical cues from degrading predictions.

The complete framework achieves state-of-the-art performance across all evaluation cohorts. In addition to attaining the highest segmentation accuracy and lowest boundary errors, the proposed method exhibited superior robustness to multi-center domain shifts and maintained stable performance in highly heterogeneous clinical data.

Beyond accuracy, the framework provides interpretable prior influence and confidence maps, enabling clinicians to understand when predictions are driven by anatomical evidence and when the model relies primarily on image appearance. This transparency may facilitate the integration of automated thrombus segmentation into clinical workflows and surgical planning.

Future works will investigate the incorporation of negative anatomical priors and the application of the proposed uncertainty-gated prior mechanism to other medical image segmentation tasks. Overall, the results indicate that patient-specific anatomical guidance combined with uncertainty-aware integration offers an effective and clinically relevant strategy for robust thrombus segmentation in real-world CTA data.
\bibliographystyle{unsrtnat}
\bibliography{references}  






\end{document}